\documentclass[12pt]{article}

\usepackage{amsmath,amssymb,amsfonts,mathtools,amsthm}
\usepackage{algorithm}
\usepackage{algpseudocode}
\usepackage{color}
\usepackage{graphicx}
\usepackage{subcaption}

\newtheorem{theorem}{Theorem}[section]

\newtheorem{corollary}[theorem]{Corollary}

\newtheorem{remark}[theorem]{Remark}

\numberwithin{equation}{section}

\newcommand{\R}{\mathcal{R}}
\newcommand{\F}{\mathcal{F}}

\newcommand{\E}{\mathcal{E}}

\newcommand{\PA}{\mathcal{P}}
\newcommand{\A}{\mathcal{A}}

\newcommand{\1}{\textbf{1}}

\newcommand{\U}{\mathcal{U}}

\newcommand{\X}{\mathcal{X}}

\newcommand{\ZA}{\mathcal{Z}}
\newcommand{\dg}{\textbf{($\dagger$)}}

\begin{document}

\begin{center}
\textbf{\Large Lagrangian Index Policy for Restless  Bandits \\ \ \\ with Average Reward }
\end{center}

\bigskip

\begin{center}
Konstantin Avrachenkov\footnote{INRIA Sophia Antipolis,
2004 Route des Lucioles, B.P.93,
06902, Sophia Antipolis Cedex, France. Email: k.avrachenkov@inria.fr. The work is supported in part by Cefipra project `LION: Learning in Operations and Networks' from Indo-French Centre for Promotion of Advanced Research (IFCPAR).}, Vivek S.\ Borkar\footnote{Department of Electrical Engineering, Indian Institute of Technology Bombay, Powai, Mumbai 400076, India. Email: borkar.vs@gmail.com. The author is supported in part by a grant from Goole Research Asia.} and Pratik Shah\footnote{Department of Mechanical Engineering, Indian Institute of Technology Bombay, Powai, Mumbai 400076, India. (Now graduate student at Georgia Institute of Technology, USA.) Email: pratik2002shah@gmail.com}
\end{center}

\bigskip

\noindent \textbf{\large Abstract:} We study the Lagrangian Index Policy (LIP) for restless multi-armed bandits with long-run average
reward. In particular, we compare the performance of LIP with the performance of the Whittle Index Policy (WIP),
both heuristic policies known to be asymptotically optimal under certain natural conditions. Even though in most cases
their performances are very similar, in the cases when WIP shows bad performance, LIP continues to perform very well. 
We then propose reinforcement learning algorithms, both tabular and NN-based, to obtain online learning schemes for 
LIP in the model-free setting.
The proposed reinforcement learning schemes for LIP require significantly less memory than the analogous schemes for WIP.
We calculate analytically the Lagrangian index for the restart model, which applies to the optimal web crawling and the minimization
of the weighted age of information. 
We also give a new proof of asymptotic optimality in case of homogeneous arms as the number of arms goes to infinity, 
based on exchangeability and de Finetti's theorem.\\


\section{Introduction}
\label{sec:intro}

\subsection{Background}

We consider restless multi-armed bandits with long-run average reward criterion as in \cite{Whittle}. These are 
Markov Decision Processes (MDPs) with two actions (`active' and `passive') coupled only through the constraint 
on the number of active actions. This model has numerous applications in resource allocation \cite{AvrachenkovAyesta,Liu}, 
queueing systems \cite{Aalto,Ansell,Duran}, web crawling and age of information \cite{Avrachenkov,Hsu,Tripathy}, clinical trials 
and other health applications \cite{Biswas,Verma,Vilar}, just to name a few. See the books \cite{GittinsBook,JackoBook,RuizBook} and the recent survey 
\cite{Nino-Mora} for thorough accounts of theory and applications of restless bandits.

The curse of dimensionality phenomenon makes the solution of the restless multi-armed bandit problem hard.
In fact, the problem is known to be PSPACE-hard \cite{Papadimitriou}, which is a class among very hard problems in the computational complexity hierarchy.
This motivated many researchers, starting from the seminal work of Whittle \cite{Whittle}, to look for efficient heuristics.
Whittle first relaxed the hard per stage constraint to an average constraint that requires only the asymptotic fraction of times the active mode was used to satisfy the constraint. This put the problem in the framework of classical constrained Markov decision processes, amenable to a Lagrange multiplier formulation. He used this to motivate the notion of subsidy, viz., a scalar $\lambda$ that is added to the running reward function when passive. He then defined indexability as the requirement that the set of states that are rendered passive under the optimal choice for the problem with subsidy, increase monotonically for each arm from the empty set to the entire state space as the subsidy is increased from $-\infty$ to $+\infty$ monotonically. If so, he assigned to each state an index, now called the {\it Whittle index}, equal to that value of subsidy for which both the active and passive actions are equally desirable for the given state. The heuristic policy he introduced, now called the Whittle Index Policy (WIP), is to order the observed states of the arms in decreasing values of the corresponding indices, settling ties according to some pre-specified rule, and then rendering top $M$ arms active, $M$ being the permitted number of active arms. While this prescription is arrived at via a relaxation of the original constraint, the resulting policy does satisfy the original per stage constraint. The policy nevertheless is a heuristic, but was shown to be asymptotically optimal by Weber and Weiss \cite{Weber} (see also \cite{Verloop})
under certain conditions---more about these conditions later---when the number of arms goes to infinity. 
This and the fact that it works very well for many problems in practice has made it a very popular heuristic.
The heuristic is also attractive because of its simplicity when the Whittle indices are explicitly known functions. There are many examples where they are, e.g., \cite{Avrachenkov}, \cite{Raghunathan}, \cite{Tripathy}, \cite{Yu}.  See also the review article \cite{Nino-Mora} for more examples.

That said, the problem still remains difficult when an explicit expression is unavailable, because the index needs a separate computation for \textit{each state}, see, e.g., \cite{Avrachenkov2}. In the original Lagrangian formulation, there is only a single variable to be learnt, viz., the Lagrange multiplier. But exact solution of the Lagrangian formulation is optimal for the relaxed constraint, not the original per stage constraint. This suggests a heuristic akin to the Whittle index, but based on the actual Lagrange multiplier, viz., to rank states at each time according to the discrepancy between the `Q-values' under active and passive actions under the Lagrange multiplier and render active the top $M$ arms. This index was introduced in \cite{Brown} in the finite horizon setting and mentioned in \cite{Gast} as a
possible heuristic for the infinite horizon with long-run average reward. We propose to call this new index the {\it Lagrangian index}. 
One more important advantage of the Lagrangian Index Policy (LIP) is that it does not require the indexability condition unlike the WIP. 

Given that the Lagrangian index is computationally much more amenable when an explicit expression for the Whittle index is not available, this makes it an attractive alternative for applications. Reinforcement learning scheme for LIP that we present here requires significantly less memory than the corresponding scheme
for WIP \cite{Avrachenkov2,Nakhleh,Pagare,Robledo1,Robledo2}.
On the theoretical side, we also have a novel proof  of independence and optimality for infinitely many homogeneous arms, based on de Finetti's theorem for exchangeable sequences of random variables. This may be of independent interest. Since LIP belongs to the family of LP-priority indices (see \cite{Verloop} for the definition and an extended treatment),
the proof of asymptotic optimality of LIP follows from the results in \cite{Verloop} under the uniform global attractor assumption. 
Our proof also uses the uniform global attractor assumption. However, it can be easily carried out for both inequality and equality constraints. 

The article is organized as follows. In the next section we formally introduce the model and notation, and define Whittle and Lagrangian indices. The next three sections are devoted to reinforcement learning algorithms for the Lagrangian index. 
In Section~\ref{sec:tabular} we present a tabular reinforcement learning algorithm for LIP and provide the proof of its convergence.
Then, in Section~\ref{sec:dqn} we describe a NN-based reinforcement learning algorithm for LIP and discuss its advantages with
respect to the analogous algorithms for WIP. 
In Section 2.3 we focus on the restart problem, which has numerous important practical applications, such as optimal web crawling and the minimization of the weighted age of information. We analytically calculate the Lagrangian index for the restart problem, which can be used for the optimal web crawling and the minimization of the weighted age of information. This then helps validate the algorithms for WIP and LIP by comparison in Section 3. We find that while both LIP and WIP yield comparable performances for the restart problem, LIP has a distinct computational edge.
We also consider a problem from literature that is provably not Whittle indexable and find that the LIP gives a better performance than the WIP. We consider the `restart problem' which has many important  practical applications. We find that while both LIP and WIP yield comparable performances for the restart problem, LIP has a distinct computational edge. Finally, 
in Section~\ref{sec:asympt} we give the second theoretical contribution of this article, viz., an alternative proof of the asymptotic optimality of LIP based on exchangeability. This may have applications going beyond the present work. 



\subsection{Model Formulation and Index Policies}
\label{sec:model}

Consider $N>1$ controlled discrete-time Markov chains ({\it arms} in bandit terminology) $X^i_n, n \geq 0,$ for $1\leq i \leq N$, 
on finite state space $\X$, controlled by the $\U := \{0,1\}$-valued control sequences $U^i_n, n \geq 0$. Given the control $U^i_n$, 
the $i$-th chain moves from the current state $X^i_n \in \X$ to the next state $y \in \X$ according to the transition probability 
kernel $p^i(y|X^i_n,U^i_n)$ and obtains the reward $r^i(X^i_n,U^i_n)$. 
We seek to solve the following long-run average reward MDP problem:\\

\noindent \textbf{($P_0$)} Maximize the long-run average reward
\begin{equation}
\lim_{n\to\infty}\frac{1}{n}E\left[\sum_{m=0}^{n-1}\sum_{i=1}^N r^i(X^i_m,U^i_m)\right], \label{MDP0}
\end{equation}
subject to the constraint
\begin{equation}
\sum_{i=1}^N U^i_n = M, \quad \forall n \geq 0, \label{exact}
\end{equation}
for a prescribed $1 \leq M < N$. The constraint (\ref{exact}) means that we need to set exactly
$M$ arms active at each time step.\\

We shall consider the Whittle relaxation \cite{Whittle} of this problem, 
where we replace \eqref{exact} by an average constraint. Thus the problem becomes\\

\noindent \textbf{($P_1$)}  Maximize the average reward
\begin{equation}
\lim_{n\to\infty}\frac{1}{n}E\left[\sum_{m=0}^{n-1}\sum_{i=1}^N r^i(X^i_m,U^i_m)\right], \label{MDP0}
\end{equation}
subject to the constraint
\begin{equation}
\lim_{n\to\infty}\frac{1}{n}E\left[\sum_{m=0}^{n-1}\sum_{i=1}^N I\{U^i_m=1\}\right] = M, \label{constraint}
\end{equation}
for a prescribed $1 \leq M < N$. \\

Here we use `$\lim$' instead of `$\liminf$' using the known fact from the theory of constrained MDPs that the former suffices under stated assumptions \cite{Altman,Ross,Piunov}. We shall refer to \eqref{MDP0} as the primary reward and \eqref{constraint} as the secondary reward. It also follows from the equivalent linear programming formulation \cite{BorkarConstrained}, (also, \cite{BorkarTopics}, section VII.1) that this problem is amenable to a Lagrangian formulation, with an equivalent unconstrained MDP (see Section 4.1, \cite{Bertsekas} for an exposition of Lagrange multipliers for equality constraints):\\

\noindent \textbf{($P_2$)}  Maximize the reward 
$$
\lim_{n\to\infty}\frac{1}{n}E\left[\sum_{m=0}^{n-1}\sum_{i=1}^N\left(r^i(X^i_m,U^i_m) + \lambda_*\left(I\{U^i_m=1\}- \frac{M}{N}\right)\right)\right], 
$$
where $\lambda_*$ is the Lagrange multiplier. \\

This separates the original MDP on $S^N$ into $N$ individual identical MDPs on $S$ that seek to maximize
\begin{equation}
\lim_{n\to\infty}\frac{1}{n}E\left[\sum_{m=0}^{n-1}\left(r^i(X^i_m,U^i_m) + \lambda_*\left(I\{U^i_m=1\}- \frac{M}{N}\right)\right)\right],
\quad i=1,...,N. \label{MDP}
\end{equation}
\\
This is possible because both the objective function and the constrained function are separable in the arms. 
We assume that these individual MDPs are irreducible under all stationary policies.

Define the `Q-values' $Q^i(x,u), \ i=1,...,N, \ x \in S, \ u \in \U$, by
$$
Q^i(x,u) = u r^i(x,1) + (1-u) (\lambda_* + r^i(x,0)) + \sum_y p^i(y|x,u) \max_v Q^i(y,v).
$$

Let us now define the {\it Lagrangian index} for the chains as follows:
\begin{equation}
\gamma^i(x) := Q^i(x,1) - Q^i(x,0). \label{Lindex}
\end{equation}
This is same as the `LP-index' of \cite{Gast}, because the LP formulation and the Lagrangian formulation for $(P_1)$ are equivalent.
The Lagrangian index also belongs to the set of {\it priority policies} defined in \cite{Verloop}.
 \\

Compare this with the Whittle index which is based on a subsidy $\lambda$ given for staying passive, and the index for state $x$ is the value of subsidy that makes the active and passive modes equally desirable. In particular, the Q-values in this case depend on the subsidy $\lambda$. If we denote this dependence explicitly by writing $Q^i_\lambda(x,u)$ for the Q-values, then the Whittle index for state $x$ is that value of $\lambda=\lambda(x)$ for which $\lambda$ is a solution of $Q^i_\lambda(x,1) - Q^i_\lambda(x,0)=0$. 
This is precisely the condition under which both control choices are equally preferred, because their Q-values coincide. This is a system of equations that needs to be solved for \textit{each} $x$. This requires a separate iterative scheme for each state $x$  unless some other method allows one to explicitly compute the Whittle indices directly. This is usually not possible or easy without some additional structural properties in the problem that facilitate this route.

In contrast, the Lagrangian index for $x$ would correspond to 
$\gamma^i(x) = Q^i_{\lambda_*}(x,1) - Q^i_{\lambda_*}(x,0)$. This makes the computation of Lagrangian index much easier.  
Specifically, it needs a single scalar `dual' iteration in the primal-dual scheme. Compare this with the situation for Whittle index, where, except for the cases such as \cite{Avrachenkov} where an explicit expression for Whittle index is available (and these are relatively very few), one needs one such iteration for \textit{each} state \cite{Avrachenkov2}. \\

\section{Learning Algorithms for Lagrangian Index}

In this section, we describe two reinforcement learning algorithms for the Lagrangian index problem. The first is tabular Q-learning which is amenable to exact analysis, but may suffer from the curse of dimensionality for large problems. Then we discuss a Deep Q-learning scheme that tries to alleviate the curse of dimensionality by using a parametrization for the $Q$ function. It lacks rigorous convergence guarantees, though it has been found to work well in practice. In the last subsection, we consider a popular test MDP for such algorithms, the restart problem, where we can explicitly calculate the Lagrangian index. Further it serves not only as a useful benchmark for our learning algorithms, but also as an example of a model based reinforcement learning algorithm.  We describe only the tabular case explicitly, a similar exercise can be carried out for the Deep Q-learning scheme.

\subsection{Tabular Q-Learning for Lagrangian Indices}
\label{sec:tabular}

In many practical situations the model ($p^i(y|x,u)$ and $r^i(x,u)$) is not known exactly or known  only approximately.
This motivates us to develop model-free learning algorithms for the Lagrangian index. We propose two  model-free reinforcement
learning algorithms: the first algorithm learns the Lagrangian index and does not enforce the hard constraint (\ref{exact}).
This algorithm can be used for instance when a simulator is available. The second algorithm respects the hard
constraint (\ref{exact}) during the learning process. Both algorithms are based on two-time scale primal-dual iterations for 
the saddle point reformulation of the Lagrangian relaxation \cite{Altman,Bertsekas}:
\begin{equation}
\label{saddle}
\lambda_* = \arg \min_\lambda \max_\varphi L(\varphi,\lambda),
\end{equation}
where 
$$
L(\varphi,\lambda) = 
\lim_{n\to\infty}\frac{1}{n}E_\varphi\left[\sum_{m=0}^{n-1}\sum_{i=1}^N\left(r^i(X^i_m,U^i_m) + \lambda \left(I\{U^i_m=1\}- \frac{M}{N}\right)\right)\right]
$$
is the Lagrangian. 

The primal iterates correspond to RVI Q-learning \cite{Abounadi} and are carried out on the faster time scale.
Namely, the Q-values are updated according to 
$$
Q_{n+1}^i(x,u)  =  Q^i_n(x,u) + a(\nu(x,u,n,i))I\{X_n = x, Z_n = u\}\times
$$
\begin{equation}
\label{QupdateA1}
\Big( u r(x,1) + (1-u) (\lambda_n + r(x,0)) + \max_{v}Q^i_n(X_{n+1},v)  - f(Q^i_n) -  Q^i_n(x,u)\Big),
\end{equation}
where $X^i_{n+1} \sim p^i(\cdot|x,u)$, $\nu(x,u,n,i)$ counts the number of times $Q^i(x,u)$ has been updated till time $n$, 
and where
\begin{equation}
f(Q) = \frac{1}{2 |\X|}\sum_{i\in S}(Q(i,0) + Q(i,1)). \label{eff}
\end{equation}
\begin{remark} Without the term involving $f(\cdot)$, the iteration would match the Q-learning algorithm for discounted reward with discount factor $1$. In this case, the dynamic programming operator $F(\cdot)$ is not a contraction. The offset $f(Q^i_n)$ is required in order to stabilize the iterates and can be shown to converge a.s.\ to the optimal cost. This is not novel, it is a feature inherited from the classical `relative value iteration' algorithm for average reward MDPs. The choice of $f(\cdot)$  is not unique, see \cite{Abounadi}. It suffices to have $f$ that is Lipschitz and satisfies: for $\1 :=$ the constant vector of all $1$'s of dimension $p+1$, $f(\1) = 1$ and for $x = [x_1, \cdots , x_{p+1}]$, $f(x + c\1) = f(x) + c$ for all $c \in \R$. See \cite{Abounadi}, Assumption 2.2. Standard examples are $f(x) = x_i$ for a fixed $i, 1 \leq i \leq p+1$, $f(x) := \min_ix_i, f(x) = \max_ix_i$ and $f(x) := \frac{1}{p+1}\sum_ix_i$.
\end{remark}  
The Lagrange multiplier is estimated by the dual iterates on the slower time scale:
\begin{equation}
\label{lambdaupdateA1}
\lambda_{n+1} = \lambda_n - \beta(n) \left(\sum_{i=1}^{N} U^i_n-M\right),
\end{equation}
where $U^i_n$ is selected according to $\epsilon$-greedy policy, i.e.,
\begin{equation}
\label{eps-greedy}
U^i_n = \left\{ \begin{array}{ll}
\arg \max_u Q^i(X^i_n,u), & \mbox{w.p. \ } 1-\epsilon,\\
\mbox{Uni}(\U), & \mbox{w.p. \ } \epsilon.
\end{array} \right.
\end{equation}
The step sizes $\{a(n)\}, \{\beta(n)\}$ need to satisfy the following conditions:
$$
\sum_n a(n) = \infty, \quad \sum_n a(n)^2 < \infty,
$$
\begin{equation}
\label{timesteps}
\sum_n \beta(n) = \infty, \quad \sum_n \beta(n)^2 < \infty, \quad \beta(n) = o(a(n)).
\end{equation}

The first algorithm is summarized in Algorithm~\ref{tab1} and its convergence is established in the following 
theorem.
\begin{theorem}\label{thm:A1}
Under the conditions (\ref{timesteps}) for time steps, 
the iterates (\ref{QupdateA1}) and (\ref{lambdaupdateA1})
converge a.s., i.e.,
$$
Q_n^i(x,u) \rightarrow Q_{\lambda_*}(x,u),
$$
and
$$
\lambda_n \rightarrow \lambda^*.
$$
as $n \rightarrow \infty$.
\end{theorem}
\begin{proof} 
The proof is based upon the `two time scale stochastic approximation' argument summarized in the Appendix. (See \cite{asyn}, Section 8.1, for an extended pedagogical treatment.) We drop the superscript `$i$' in $Q^i_n$ and elsewhere  for notational ease in this proof.  To map the foregoing to that paradigm, observe that the iterates (3.2) for $\{Q_n\}$ are on a faster time scale governed by the step sizes $\{a(n)\}$, whereas the iterates (3.4) for $\{\lambda_n\}$ are on a slower time scale dictated by the slower step sizes $\{\beta(n)\}$.  Let $q = \{q_{i,u}, i \in S, u \in  \U\},$ vectorized using the lexicographical order so that $q \in \R^p$ where $p := 2|\X|$.
The $\{Q_n\}$-iterate can be rewritten as
$$Q_{n+1}(x,u) = Q_n(x,u) + a(\nu(x,u,n))I\{X_n=x,U_n=u\}\left(F_{i,u}(Q_n) + M_{i,u}(n+1)\right).$$
Here $F = \{F_{i,u}\}: \R^p \to \R^p$, vectorized lexicographically,  is given by
$$F_{x,u}(q, \lambda) := ur(x,1) + (1-u)\left[\lambda + r(x,0) + \sum_jp(y|x,u)\max_vq_{y,v} - f(q)\right],$$
where $f(\cdot)$ is an offset required to stabilize the iterates.
%
Rewrite \eqref{lambdaupdateA1} as
\begin{equation}
\label{lambdaupdateA*}
\lambda_{n+1} = \lambda_n - a(n)\left(\frac{\beta(n)}{a(n)}\right) \left(\sum_{i=1}^{N} U^i_n-M\right),
\end{equation}
Consider the time scale corresponding to stepsizes $\{a(n)\}$. As described in the Appendix,  $\{(Q_n, \lambda_n)\}$ will a.s.\ track the asymptotic behaviour of the ODE in $\R^p\times\R$ given by
\begin{equation}
\dot{q}(t) = F(q(t), \lambda) - q(t),  \ \dot{\lambda} = 0,\label{FPode}
\end{equation}
where the second equation follows from the fact that $\beta(n) = o(a(n))$. Hence  $\lambda$ can be treated as a fixed parameter, following the theory of two time scale and asynchronous stochastic approximation described in the Appendix. This is precisely the model studied in Section 3 of \cite{Abounadi}. By Lemma 3.8 of \cite{Abounadi}, we have 
\begin{equation}
\sup_n\|Q_n\| < \infty \label{qunibound}
\end{equation} 
a.s.\ Note also that $F(\cdot, \lambda)$ satisfies \eqref{twoscalebound}, from which \eqref{vanishingtail} follows by the martingale convergence theorem as pointed out in the Appendix.
 By Theorem 3.4 of \cite{Abounadi}, we have a.s.\ convergence of $\{Q_n\}$ to the unique solution $Q^\lambda \in \R^M$ of the fixed point equation $F(Q, \lambda) = Q$. (Note that the proof of this in \cite{Abounadi} is very involved.)

Next consider the scalar iterate $\{\lambda_n\}$ on the slow time scale. By the two time scale argument, 
it will satisfy the ODE
$$\dot{\lambda}(t) = g(\lambda(t)) - \lambda(t),$$
where $g(\lambda) :=$ the stationary expectation of $\sum_nU_n - M$.  From \eqref{eps-greedy} and using the fact that $\lambda_n$ is on the slower time scale, hence quasi-static on the fast time scale, we have
$$g(\lambda) = \sum_i\pi_\lambda(i)\left((1-\epsilon)argmax_uQ_\lambda(i,u) + \frac{1}{2}\epsilon\right),$$
where $\pi_{\lambda}$ denotes the stationary distribution of the Markov chain controlled by $\{U_n\}$ as in \eqref{eps-greedy} for $\lambda_n \equiv \lambda$.  (We use here the fact that $\{Q_n\}$ being iterated on a faster time scale leads to $Q_n(i,u) \approx Q_{\lambda_n}(i,u)$.) Note that $g(\cdot)$ is a bounded function, so the RHS of the above ODE will be $< 0$ (resp., $> 0$) for $\lambda \gg 0$ (resp., $\lambda \ll 0$). (A similar argument applied to the iterates themselves, along with \eqref{qunibound} above,  ensures the validity of \eqref{a.s.bound} in this context, i.e., $\sup_n| \lambda_n| < \infty$ a.s.)  A bounded trajectory of a scalar well-posed ODE must converge, because uniqueness dictates that it must be monotone---its direction of motion away from any equilibrium is fixed at every point by the vector field at the  point whence it cannot traverse the point in both directions. This implies that it will converge to a solution of $g(\lambda) = \lambda$. On the other hand, recalling our Lagrange multiplier formulation, the $\lambda$-iterate is nothing but a stochastic subgradient descent for minimizing $\max_\varphi L(\varphi, \cdot)$, which is convex. Thus this is a subgradient descent on a convex, in fact a piecewise linear function.
Along the lines of \cite{Avrachenkov2}, where a very similar situation arises, the iterates (\ref{lambdaupdateA1}) remain bounded 
and converge a.s.\ to a global minimum.

\end{proof}

\begin{algorithm}[H] 
\caption{RL for Lagrangian index with relaxed constraint}
\label{tab1}
\begin{algorithmic}[1]
    \State {\bf initialize} $\lambda$, Q(x,u) for all states and actions, $\epsilon$
    \For {n\;=\;1:\;$n_{end}$}
        \State Choose action $U_{n}^i$ for each arm $i$ in $\epsilon$-greedy fashion according to (\ref{eps-greedy})
        \State Update $X_{n+1}^i$ and reward $r_n^i$ from $X_{n}^i$ and $U_{n}^i$ for every arm $i$
        \State Update ($X_n^i, U_n^i$)-th Q-value for each arm $i$ as in (\ref{QupdateA1})
        \State Update common subsidy for passivity $\lambda$ as in (\ref{lambdaupdateA1})
        \State Optional: decrement $\epsilon$ as $\epsilon_{n+1}=max(0.01,\epsilon_{n}*0.99)$
    \EndFor
    \State Calculate the new estimate of the index for each arm $i$:
        \Statex\hspace*{5mm} $\gamma^i_{n+1}(X^i_n)=Q^i_n(X^i_n,1)-Q^i_n(X^i_n,0)$
\end{algorithmic}
\end{algorithm}

In some cases a simulator may not be available and experiments with a real system require us 
to respect the hard constraint (\ref{exact}). In such cases, we propose a modification 
of Algorithm~\ref{tab1} described as Algorithm~\ref{tab2} by using virtual actions. The virtual actions 
are used to estimate the Lagrange multiplier and the actual actions are generated according 
to the Lagrangian indices. This way the algorithm respects the constraint on the resources
even during the training phase.

\begin{algorithm}[H] 
\caption{RL for Lagrangian index with hard constraint}
\label{tab2}
\begin{algorithmic}[1]
    \State {\bf initialize} $\lambda$, Q(x,u) for all states and actions, $\epsilon$
    \For {n\;=\;1:\;$n_{end}$}
        \State Choose action $U_{n}^i$ according to $\epsilon$-Lagrangian policy, i.e.,
        \Statex\hspace*{5mm} w.p. $1-\epsilon$, activate top-$M$ arms according to Lagrangian index
        \Statex\hspace*{5mm} w.p. $\epsilon$, activate $M$ random arms
        \State Choose virtual action $\tilde{U}_{n}^i$ for each arm $i$ in $\epsilon$-greedy fashion as in (\ref{eps-greedy})
        \State Update $X_{n+1}^i$ and reward $r_n^i$ from $X_{n}^i$ and $U_{n}^i$ for every arm $i$
        \State Update ($X_n^i, U_n^i$)-th Q-value for each arm $i$ as in (\ref{QupdateA1})
        \State Update common subsidy for passivity $\lambda$ using virtual actions, i.e., 
        \Statex\hspace*{5mm} $\lambda_{n+1} = \lambda_n - \beta(n) \left(\sum_{i=1}^{N}\tilde{U}^i_n-M\right)$
        \State Optional: decrement $\epsilon$ as $\epsilon_{n+1}=max(0.01,\epsilon_{n}*0.99)$
    \EndFor
    \State Calculate the new estimate of the index for each arm $i$:
        \Statex\hspace*{5mm} $\gamma^i_{n+1}(X^i_n)=Q^i_n(X^i_n,1)-Q^i_n(X^i_n,0)$
\end{algorithmic}
\end{algorithm}

\begin{theorem}\label{thm:A2}
Under the conditions (\ref{timesteps}) for time steps, 
the iterates of Algorithm~2 converge a.s., i.e.,
$$
Q_n^i(x,u) \rightarrow Q_{\lambda_*}(x,u),
$$
and
$$
\lambda_n \rightarrow \lambda^*.
$$
as $n \rightarrow \infty$.
\end{theorem}

\begin{proof} The difference between Algorithms 1 and 2 is in the sampling scheme for $U_n, n \geq 0$. 
The only requirement for our convergence argument where sampling plays a key role is that
$$\liminf_{n\to\infty}\frac{\nu(x,u,n)}{n} > 0 \ \mbox{a.s.}$$
This is ensured in both cases by the irreducibility condition on the Markov chain and the $\epsilon$-randomization of controls, which are common to both schemes. The rest of the proof is the same as in  Theorem~\ref{thm:A1}.
\end{proof}

We would like to note that Algorithms~\ref{tab1}~and~\ref{tab2} are significantly simpler and use
less memory than the reinforcement learning algorithm in \cite{Avrachenkov2} for the Whittle
index. This is due to the fact that we do not need to use a reference state and many copies
of the Q-table.

\subsection{Deep Q-Learning for Lagrangian Indices}
\label{sec:dqn}

Till now we have used tabular Q-learning for learning Q-values and ultimately the Lagrangian indices. 
Now, let us use the Deep Q-Learning Network (DQN) \cite{Minh1,Mnih2} for approximation of the Q values. 
This will especially be useful in cases where the state space is large. A complete scheme of how our 
DQN algorithm works can be found in Algorithm~\ref{tab:DQN}. We note that in comparison with the DQN schemes for WIP 
\cite{Pagare,Robledo1,Robledo2}, the DQN for LIP has a much simpler architecture and as a result is numerically more stable 
and simpler to tune. This is because the DQN algortihms for WIP \cite{Pagare,Robledo1,Robledo2} require extra input for 
the reference state (recall that Whittle index is a function of the reference state in addition to the current state of the Markov chain.). 

In the case of homogeneous arms, one-dimensional input to NN suffices, which represents the state of the arm. 
In the case of heterogeneous arms, a two-dimensional input to NN is used. The first dimension is used 
to give information about the type of arm, and the second dimension specifies the state of the arm. 
The neural network outputs are the Q-values for both possible actions. The actions are taken in 
an $\epsilon$-greedy fashion. Namely, the probability for exploration $\epsilon$ is initialized 
as one and then, at the end of every iteration, is reduced by the exploration-decay factor. 
The min value of $\epsilon$ is set to be 0.01.

The equation for the Q target used in training DQN is:
$$
    Q_{target}(X_n,U_n)\gets((1-a_n)(r_0(X_n)+\lambda_n)+a_nr_1(X_n)
$$    
\begin{equation}
\label{Qtarget}    
    +\underset{v\in\{0,1\}}{max}Q_n(X_{n+1},v)-\frac{1}{2d}\sum\limits_{k\in S}(Q_n(k,0)+Q_n(k,1)).
\end{equation}
Note that in DQN we employ a second neural network, called the target model, for the computation of $\max_{v\in\{0,1\}}Q_n(X_{n+1},v)$. 
The reason for this decision is due to the maximisation bias of Q-learning: overestimating the Q-values results in the error increasing  
over time, due to the target being $r + \gamma \max_a Q(i,a)$. The use of a second neural network for the target helps control this bias. 
This second neural network copies periodically the parameter values of the main network.

We store the tuples $(X_{n}^i,U_{n}^i,r_{n}^i,X_{n+1}^i,\lambda_n)$ in an experience replay buffer, from which a batch of random samples 
is taken every time to train the main neural network. For each tuple $(X_{n}^i,U_{n}^i,r_{n}^i,X_{n+1}^i,\lambda_n)$ we calculate Q target 
values using equation (\ref{Qtarget}). Once we have calculated Q-values for all the examples in batch through $Q(\cdot)$ and 
its targets $Q_{target}(\cdot)$, we compute the loss function as the mean square error between the two and train the main neural network, 
using a standard Adam optimiser.

\begin{algorithm}[H]
\caption{DQN Algorithm for Lagrangian Index}
\label{tab:DQN}
\begin{algorithmic}[1]
    \State {\bf initialize} $\lambda^*$, DQN policy, DQN target, Batch Size, $\epsilon$=1

    \For {n\;=\;1:\;$n_{end}$}
        
        \State Update $\beta(n)$ as:
        $\beta(n)=\frac{1}{\lceil{n\log(n)/ 5000}\rceil+1}$
        \For{arm $i\in N$}
            \State Choose action $U_{n}^i$ for each arm i in an $\epsilon$-greedy fashion
            \State Update $X_{n+1}^i$ and reward $r_n^i$ from $X_{n}^i$ and $U_{n}^i$ for every arm i
            \State Store $(X_{n}^i,U_{n}^i,r_{n}^i,X_{n+1}^i,\lambda_n)$ in experience replay buffer
            \State Train DQN policy model on one batch of $(X_{n}^i,U_{n}^i,r_{n}^i,X_{n+1}^i,\lambda_n)$ tuples from experience replay buffer
        \EndFor
        \State Update the DQN target model
        \State Update common subsidy for passivity for all arms $\lambda$ as: \\ $\lambda_{n+1}=\lambda_n+\beta(n)\left(\sum\limits_{i=1}^{i=N}U_n^i-M\right)$
        
        \State Decrement $\epsilon$ as $\epsilon_{n+1}=max(0.01,\epsilon_{n}*0.99)$

    \EndFor
    \State Calculate the new index for every state s of each arm i by:
        \Statex\hspace*{5mm} $\gamma^i(x)=Q_{\theta}(x,1)-Q_{\theta}(x,0)$
\end{algorithmic}
\end{algorithm}

\subsection{Application to Restart Model}
\label{sec:restart}

Let us consider the following MDP, which can describe both the optimal 
web crawling \cite{Avrachenkov} and the minimization of the weighted age of information \cite{Hsu,Sombabu}. 

There are $N$ information sources, which can be viewed as arms of RMAB. The central device can 
probe $M$ information sources at a time. If the central device probes source $i$, it succeeds with
probability $p_i$ (the failure can be due to either the distant server timeout or due to the loss of 
the probing packet). The freshness of the information at source $i$ at time slot $n$ is modelled 
by $X^i_n$ with the following dynamics:
\begin{equation}
\label{eq:fresh_dynamics}
X^i_{n+1}=
\left\{\begin{array}{ll}
1, & \mbox{if \ } U^i_n=1 \mbox{ \ and the probe is successful,}\\
X^i_n + 1, &  \mbox{otherwise}.
\end{array}\right.
\end{equation}
For the relaxed problem formulation \textbf{($P_2$)}, we consider the following immediate
reward
\begin{equation}
\label{eq:fresh_reward}
r(x,u) = -w_i x + \lambda u,
\end{equation}
where $w_i$ gives the factor of importance of source $i$.
We shall again use the saddle point reformulation (\ref{saddle}) of the Lagrangian relaxation.
Since we have moved the constraint to the Lagrangian, the optimal policy can be taken to be deterministic. In fact, it is proven in \cite{Hsu} (see also a similar development in \cite{Avrachenkov})
that the optimal policy for the unconstrained MDP with the reward (\ref{eq:fresh_reward})
is of threshold type\footnote{The threshold property is usually proved by using additional structural properties such as monotonicity, convexity, submodularity, etc.\ of the value function. The required properties are satisfied in the present model.}. Therefore, we can search for an optimal policy in the form:
\begin{equation}
\label{thres_policy}
\varphi^i_*(1|x)=
\left\{\begin{array}{ll}
0, & \mbox{ \ if \ } x < \bar{x}_i,\\
1, & \mbox{ \ if \ } x \ge \bar{x}_i.
\end{array}\right.
\end{equation}
To analyse the model, we use the renewal argument dividing the time into cycles.
A new cycle starts when $X^i_n=1$. The expected duration of cycle is given by
\begin{equation}
\label{eq:Tcycle}
E_{\varphi^i_*}[T^i_{cycle}] = \bar{x}_i - 1 + \frac{1}{p_i}.
\end{equation}
Let us calculate the average gain $g^i$ corresponding to the policy $\varphi^i_*$.
By the renewal argument, we can write
$$
g^i = \frac{E_{\varphi^i_*}[R^i]}{E_{\varphi^i_*}[T^i_{cycle}]},
$$
where $E_{\varphi^i_*}[R^i]$ is the expected reward gained by arm $i$ on a cycle. (We suppress the dependence of $g$ on $\varphi^i_*$ for notational ease.) 
Next we calculate this by observing that a cycle can have two stages. Namely,
$$
E_{\varphi^i_*}[R^i] = E_{\varphi^i_*}[R^i_1] + E_{\varphi^i_*}[R^i_2],
$$
where
$$
E_{\varphi^i_*}[R^i_1] = -w_i \frac{(\bar{x}_i-1)\bar{x}_i}{2},
$$
corresponds to the first deterministic stage when $X^i_n < \bar{x}_i$,
and
$$
E_{\varphi^i_*}[R^i_2] = E_{\varphi^i_*}\left[-w_i\frac{T_2(2\bar{x}_i+T_2-1)}{2} + \lambda T_2\right]
$$
$$
= -w_i \left(\frac{\bar{x}_i}{p_i}+\frac{1}{2}\frac{2-p_i}{p_i^2} - \frac{1}{2p_i}\right) + \frac{\lambda}{p_i}
$$
$$
= -w_i \left( \frac{\bar{x}_i-1}{p_i}+\frac{1}{p_i^2} \right) + \frac{\lambda}{p_i},
$$
where we used the fact that the duration of the second stage $T_2$ is geometrically
distributed. Thus,
$$
g^i =
\frac{-w_i \frac{(\bar{x}_i-1)\bar{x}_i}{2} - w_i \left( \frac{\bar{x}_i-1}{p_i}+\frac{1}{p_i^2} \right) 
+ \frac{\lambda}{p_i}}{\bar{x}_i - 1 + \frac{1}{p_i}}.
$$
To optimize the above expression, we first find the minimizer of the continuous relaxation
by equating to zero the derivative of $g^i$. This minimizer is given as the positive root 
of the following quadratic equation:
$$
p_i w_i \bar{x}_i^2 + 2 w_i (1-p_i) \bar{x}_i + (2\lambda - w_i(1-p_i)) = 0,
$$
which can be explicitly written as
$$
\tilde{x}_i = \frac{\sqrt{(1-p_i)-2\lambda p_i/w_i}-(1-p_i)}{p_i}.
$$
Since $\bar{x}_i$ is an integer, the maximal value of $g_i$ is given by
\begin{equation}
\label{eq:gainopt}
g^i_{opt} = \max \{ g^i(\lfloor \tilde{x}_i \rfloor), g^i(\lceil\tilde{x}_i\rceil) \},
\end{equation}
where $\lfloor a \rfloor$ is the integer part of $a$ and $\lceil a \rceil = \lfloor a \rfloor + 1$. We shall denote by $\bar{x}^i_{opt}$ either $\lfloor \tilde{x}_i \rfloor$ or $\lceil\tilde{x}_i\rceil$ depending on whether the maximum above is attained at the former or at the latter. 

Now, by (\ref{saddle}), finding $\lambda$ is an easy one-dimensional convex optimization problem:
\begin{equation}
\label{eq:lambdaopt}
\lambda_* = \arg \min_\lambda L(\phi_{opt},\lambda) = \arg \min_\lambda \sum_{i=1}^{N} g^i_{opt}(\lambda) - \lambda \alpha N.
\end{equation}
This problem can be efficiently solved for instance by the bisection method.

An interesting alternative to the solution of the above optimization problem
is to estimate the Lagrange multiplier online as in (\ref{lambdaupdateA1}).
The advantage of this approach is that we effectively use the knowledge about
the structure of the model and do not need to perform more demanding Q-updates (\ref{QupdateA1}).

Next, we need to satisfy the constraint (\ref{constraint}). Again by the the renewal argument,
we have 
$$
\lim_{n\to\infty} \frac{1}{n} E_{\varphi^i_*} \left[ \sum_{m=0}^{n-1} I\{U^i_m=1\} \right]
= \frac{1}{E_{\varphi^i_*}[T^i_{cycle}]} = \frac{1}{\bar{x}_i - 1 + 1/p_i}.
$$
In general, we shall have
$$
\sum_{i=1}^{N} \frac{1}{\bar{x}_{i,opt}(\lambda_*) - 1 + 1/p_i} \neq M.
$$
However, from the theory of the constrained MDP (see e.g., \cite{Altman, BorkarConstrained, BorkarConstraint, Piunov}), we
know that in this case there will be at least on arm $i_*$ such that 
$$
g^{i_*}(\lfloor \tilde{x}_{i_*} \rfloor) = g^{i_*}(\lceil\tilde{x}_{i_*}\rceil).
$$
This allows us to do a randomization of the policy at that arm and to satisfy
the constraint (\ref{constraint}) exactly. We would like to note that in the
case of a large number of arms, the discrepancy is very small and it makes a 
very small practical difference if we perform such randomization or just accept
a minor violation of the constraint. Furthermore, as we see shortly, the expression
of the Lagrangian index depends only on the optimal gain and the optimal Lagrange
multiplier and does not require the explicit expression of the optimal policy. 

Specifying the expression (\ref{Lindex}), we can write
$$
\gamma^i(x) = Q^i(x,1) - Q^i(x,0)
$$
$$
=\left(-w_i x + \lambda_* - g^i_{opt} + p_i V^i(1) + (1-p_i) V^i(x+1)\right) - \left(-w_i x - g^i_{opt} + V^i(x+1)\right)
$$
$$
=\lambda_* + p_i\left(V^i(1)-V^i(x+1)\right).
$$
Next we obtain $V$-values from the Bellman equation
$$
V^i(x) = \max_u \left\{ -w_i x + \lambda_* u - g^i_{opt} + \sum_y p^i(y|x,u) V^i(y) \right\}.
$$
Since $V^i(\cdot)$ is the relative value, we can set $V^i(1)=0$ which simplifies calculations. 
We use the Bellman equation for two cases separately:
\begin{enumerate}
\item Case $x < \bar{x}_{i,opt}$: $V^i(x) = -w_ix -g^i_{opt} + V^i(x+1)$ or $V^i(x+1)=V^i(x)+w_i x +g^i_{opt}$.
Consequently, we have
$$
V^i(x) = V^i(1) + g^i_{opt} (x-1) + w_i \frac{(x-1)x}{2} = g^i_{opt} (x-1) + w_i \frac{(x-1)x}{2}.
$$
\item Case $x \ge \bar{x}_{i,opt}$: $V^i(x) = -w_ix +\lambda -g^i_{opt} +p_i V^i(1)+(1-p_i)V^i(x+1)$, and hence
$V^i(x+1) = (V^i(x) + w_ix -\lambda +g^i_{opt})/(1-p_i)$. This is a linear difference equation,
whose general solution is given by
$$
V^i(x) = \frac{C_i}{(1-p_i)^x} -\frac{w_i}{p_i}x +\frac{p_i(\lambda_*-g^i_{opt})-w_i(1-p_i)}{p_i^2},
$$
where the constant $C_i$ is found from the condition 
$$
V(\bar{x}_{i,opt}) = g^i_{opt} (\bar{x}_{i,opt}-1) + w_i \frac{(\bar{x}_{i,opt}-1)\bar{x}_{i,opt}}{2}.
$$
\end{enumerate}
Thus, we can summarize the above derivations in the following theorem and Algorithm~\ref{tab:restart}
for online computation of the Lagrangian indices in the restart model. Note that in this version of the
algorithm the probabilities $p_i$, $i=1,...,N$, are assumed to be known. However, a simple and efficient online
estimator can be added if they need to be estimated as well. In fact, the problem of efficient estimation 
of change rates of information sources has been recently addressed in \cite{Gugan1,Gugan2}. 
We would like to emphasize that in the design of reinforcement learning algorithms it is
always preferable to take into account the known structure of the model.
\begin{theorem}
\label{thm:LIrestart}
The Lagrangian index for the restart model has the following expression:
$$
\gamma^i(x) = \lambda_* - p_i V^i(x+1),
$$
where
$$
V^i(x) = g^i_{opt} (x-1) + w_i \frac{(x-1)x}{2}, \quad x < \bar{x}_{i,opt},
$$
$$
V^i(x) = \frac{C_i}{(1-p_i)^x} -\frac{w_i}{p_i}x +\frac{p_i(\lambda_*-g^i_{opt})-w_i(1-p_i)}{p_i^2},
\quad x \ge \bar{x}_{i,opt},
$$
and where the Lagrange multiplier $\lambda_*$ is found from the 
one-dimensional convex optimization (\ref{eq:lambdaopt}) and 
the optimal gain $g^i_{opt}$ is given by (\ref{eq:gainopt}).
\end{theorem}

\begin{algorithm}[H]
\caption{Algorithm for Restart Problem}
\label{tab:restart}
\begin{algorithmic}[1]
    \State {\bf initialize} $\lambda=0$, $\epsilon=0.01$

    \For {n\;=\;1:\;$n_{end}$}
        \For{arm $i\in N$}
        \State Calculate $
\tilde{x}_i = \frac{\sqrt{(1-p_i)-2\lambda p_i/w_i}-(1-p_i)}{p_i}.
$
        \State Get $g^i_{opt} = \max \{ g^i(\lfloor \tilde{x}_i \rfloor), g^i(\lceil\tilde{x}_i\rceil) \}$ using
        \Statex $
\;\;\;\;\;\;\;\;\;\;\;g^i =
\frac{-w_i \frac{(x_i-1)x_i}{2} - w_i \left( \frac{x_i-1}{p_i}+\frac{1}{p_i^2} \right) 
+ \frac{\lambda}{p_i}}{x_i - 1 + \frac{1}{p_i}},
$ substituting $x_i=\lfloor \tilde{x}_i \rfloor$ and 
\Statex \;\;\;\;\;\;\;\;\;\;\; $x_i=\lceil\tilde{x}_i\rceil $
\State $\bar{x}^i_{opt}$ is either $\lfloor \tilde{x}_i \rfloor$ or $\lceil\tilde{x}_i\rceil$ depending on whether the maximum
\Statex \;\;\;\;\;\;\;\;\;\;\;above is attained at the former or at the latter
\State $Q^i(x,1)=-w_i x + \lambda - g^i_{opt} +  (1-p_i) V^i(x+1)$
\Statex$\;\;\;\;\;\;\;\;\;\;\;Q^i(x,0)=-w_i x - g^i_{opt} + V^i(x+1)$ where 
\Statex$\;\;\;\;\;\;\;\;\;\;\;
V^i(x) = g^i_{opt} (x-1) + w_i \frac{(x-1)x}{2}, \quad x < \bar{x}_{i,opt},
$
\Statex$\;\;\;\;\;\;\;\;\;\;\;
V^i(x) = \frac{C_i}{(1-p_i)^x} -\frac{w_i}{p_i}x +\frac{p_i(\lambda-g^i_{opt})-w_i(1-p_i)}{p_i^2},
\quad x \ge \bar{x}_{i,opt},
$
    \State We get $C_i$ from $
 \frac{C_i}{(1-p_i)^{\bar{x}_{i,opt}}} -\frac{w_i}{p_i}{\bar{x}_{i,opt}}+\frac{p_i(\lambda-g^i_{opt})-w_i(1-p_i)}{p_i^2} = $ 
 \Statex \;\;\;\;\;\;\;\;\;\;\;$V(\bar{x}_{i,opt}) =g^i_{opt} (\bar{x}_{i,opt}-1) + w_i \frac{(\bar{x}_{i,opt}-1)\bar{x}_{i,opt}}{2}.
$

            \State Choose action $a_{n}^i$ for each arm $i$ in an $\epsilon$-greedy fashion
            based on
            \Statex \;\;\;\;\;\;\;\;\;\;\;$Q^i(x,1)$ and $Q^i(x,0)$
        \EndFor
       
        \State Update $\lambda$ as: $\lambda_{n+1}=\lambda_n-\beta(n)\left(\sum\limits_{i=1}^{i=N}U_n^i-M\right)$

    \EndFor
   \State $
\gamma^i(x) = \lambda - p_i\left(V^i(x+1)\right)
$
\end{algorithmic}
\end{algorithm}

\section{Numerical Examples}
\label{sec:num}

In this section we provide numerical experiments for the restart problem (where the Lagrangian index is explicitly known), a problem from literature that is provably not Whittle indexable, and the deadline scheduling problem. This help us illustrate different facets of the Lagrangian index.

\subsection{Restart Model}

In this first numerical experiment, we consider the restart model as described in Section~\ref{sec:restart} and
apply Algorithm~\ref{tab:restart}. There are four different categories of arms: 
\begin{itemize}
    \item Type 1: $p_1=0.95, w_1=0.9$ (reliable, important)
    \item Type 2: $p_2=0.95, w_2=0.2$ (reliable, not important)
    \item Type 3: $p_3=0.7, w_3=0.95$ (not very reliable, important)
    \item Type 4: $p_4=0.7, w_4=0.2$  (not very reliable, not important)  
\end{itemize}
A total of 100 arms are taken, with 25 belonging to each category and the budget is set as 16 arms. 
Therefore, in this example $N=100$ and $M=16$. Note that for all the 100 arms we still learn just one common Lagrange multiplier.

Figure~\ref{fig:RestartSubsidy}(a) shows the convergence of the subsidy (Lagrange multiplier) to the value -11.6. In the right Figure~\ref{fig:RestartSubsidy}(b) we plot the value of the optimal Lagrange function (optimized with respect to the policy) at different values of the Lagrange multiplier and as expected we see that the Lagrange function is piecewise linear and takes the minimum value around -11.6 (shown by the red line), same as the value of the Lagrange multiplier we obtain by our Algorithm~4.

\begin{figure}[H]
  \centering

  \begin{subfigure}[t]{0.45\textwidth}
    \centering
    \includegraphics[width=\textwidth]{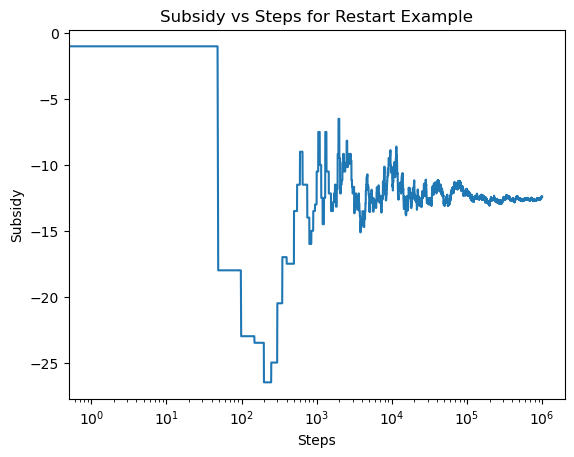}
    \caption{Subsidy (Lagrange multiplier) estimation.}
    \label{fig:RestartSubsidy1}
  \end{subfigure}
  \begin{subfigure}[t]{0.45\textwidth}
    \centering
    \includegraphics[width=\textwidth]{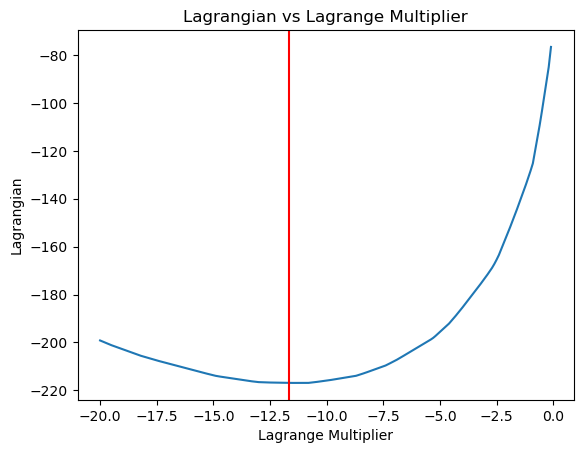}
    \caption{Optimal Lagrange multiplier $\lambda_*$.}
    \label{fig:RestartSubsidy2}
  \end{subfigure}
  \caption{Restart Model. Subsidy (Lagrange multiplier).}
  \label{fig:RestartSubsidy}
\end{figure}

In Figure~\ref{fig:RestartReward} we compare the performances of LIP with WIP. We used
the Whittle indices calculated in \cite{Sombabu}. As we see in the restart example LIP 
performs as good as WIP.

\begin{figure}[H]
    \centering
    \includegraphics[scale=0.5]{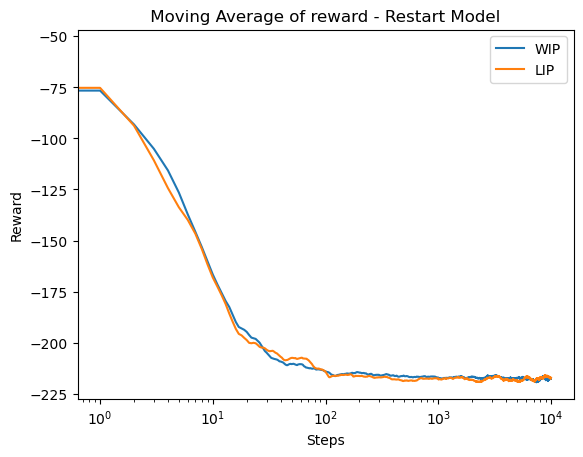}
    \caption{Restart Model. LIP and WIP average reward comparison.}
    \label{fig:RestartReward}   
\end{figure}

\subsection{A Non Whittle-Indexable Problem}

In this example, we test the performance of the Lagrangian Index in a non-Whittle Indexable Problem as described in Appendix A.4 of \cite{Gast0}.

The transition probability matrices for this problem are given by
\begin{align*}
P_0=
\begin{bmatrix}
0.005 & 0.793 & 0.202\\
0.027 & 0.558 & 0.415 \\
0.736 & 0.249 & 0.015
\end{bmatrix}
\quad \text{and} \quad
P_1=
\begin{bmatrix}
0.718 & 0.254 & 0.028\\
0.347 & 0.097 & 0.556 \\
0.015 & 0.956 & 0.029
\end{bmatrix}
\end{align*}
and the rewards are given by
$$
r(1,0)=0, \quad r(2,0)=0, \quad r(3,0)=0,
$$
$$
r(1,1)=0.699, \quad r(2,1)=0.362, \quad r(3,1)=0.715.
$$
In many practical cases even when the problem is not Whittle indexable, Whittle indices 
are used empirically. Hence, we compute online the Whittle indices for this problem
by the algorithm from \cite{Avrachenkov2} and compare the performance of Lagrangian indices 
against them. We test both Algorithm~\ref{tab1} and Algorithm~\ref{tab2} on this example,
see Figure~\ref{fig:NonWhittle1} and Figure~\ref{fig:NonWhittle2}, respectively. 
As seen from the plots, the Lagrangian indices result 
in a much higher reward as compared to the reward obtained by the algorithm estimating Whittle indices. Also, we note that even though Algorithm~\ref{tab2} respects the hard constraint
and uses virtual actions to estimate subsidy, this does not affect the rate of convergence
of the moving average reward.

\begin{figure}[H]
  \centering

  \begin{subfigure}[t]{0.45\textwidth}
    \centering
    \includegraphics[width=\textwidth]{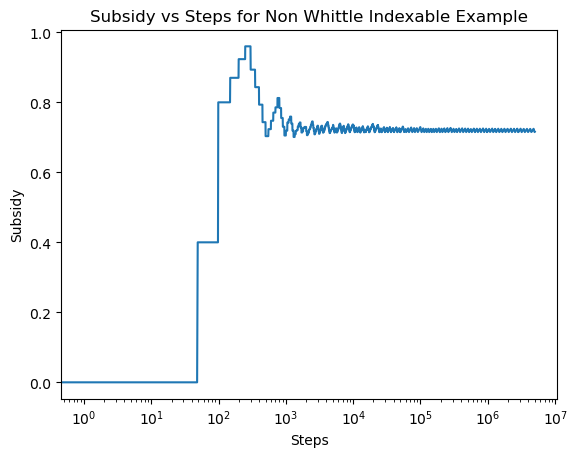}
    \caption{Subsidy (Lagrange multiplier) estimation.}
    \label{fig:RestartSubsidy1}
  \end{subfigure}
  \begin{subfigure}[t]{0.45\textwidth}
    \centering
    \includegraphics[width=\textwidth]{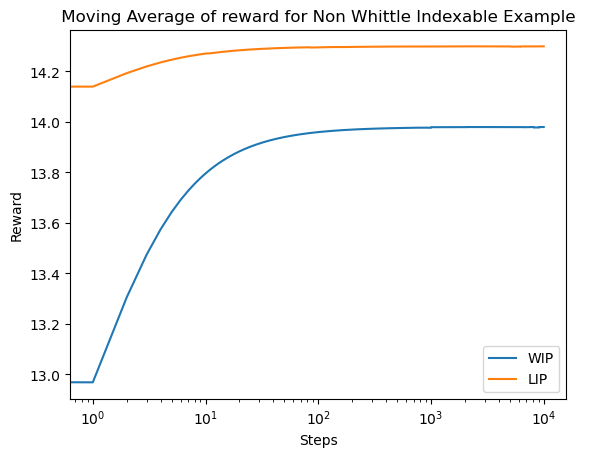}
    \caption{Moving average reward.}
    \label{fig:RestartSubsidy2}
  \end{subfigure}
  \caption{Non Whittle Indexable problem. Algorithm~\ref{tab1}.}
  \label{fig:NonWhittle1}
\end{figure}

\begin{figure}[H]
  \centering

  \begin{subfigure}[t]{0.45\textwidth}
    \centering
    \includegraphics[width=\textwidth]{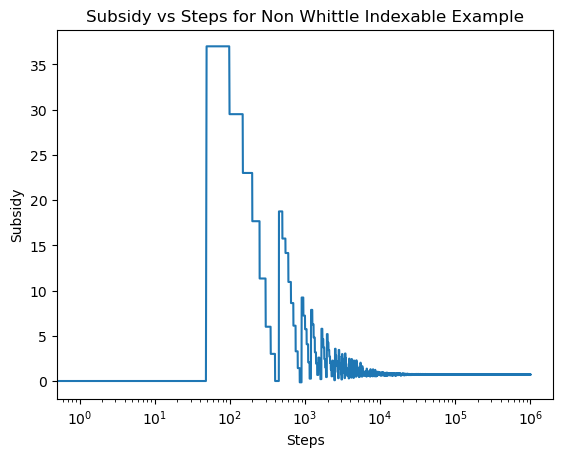}
    \caption{Subsidy (Lagrange multiplier) estimation.}
    \label{fig:RestartSubsidy1}
  \end{subfigure}
  \begin{subfigure}[t]{0.45\textwidth}
    \centering
    \includegraphics[width=\textwidth]{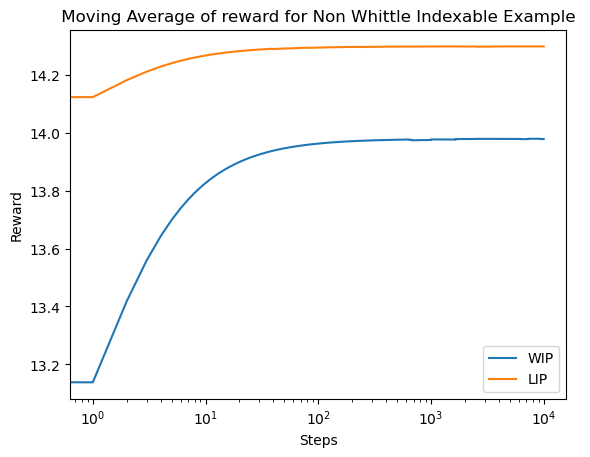}
    \caption{Moving average reward.}
    \label{fig:RestartSubsidy2}
  \end{subfigure}
  \caption{Non Whittle Indexable problem. Algorithm~\ref{tab2}.}
  \label{fig:NonWhittle2}
\end{figure}

\subsection{Deadline Scheduling Problem}

Let us now consider the deadline scheduling problem studied in \cite{Yu}. The states of this
problem are formed by two different variables: the service time $B \in [0, 9]$ and
the deadline $T \in [0, 12]$, leading to a total of $|S| = 130$ states (although several 
of these states are not accessible by the underlying Markov chain). The variable $B$ 
represents the amount of workload pending to complete a certain job while
the variable $T$ represents the remaining time available to perform it. When there is 
no job at the $i^{th}$ position, the state is $(0, 0)$, while otherwise it is 
$(T_n^i, B_n^i)$. In each iteration, the state transition from $X_n^i=(T_n^i, B_n^i)$ 
depends on the action $U_n^i$ in the following way:
\begin{equation}
s_{n+1}^i =
\begin{cases} 
(T_n^i - 1, (B_n^i - U_n^i)^+), & \text{if } T_n^i > 1, \\
(T, B) \text{ with prob. } q(T, B), & \text{if } T_n^i \leq 1.
\end{cases}
\end{equation}
where $b^+ = \max(b, 0)$. When $T$ = 1, the deadline for completing the job ends, 
and the new state (including the empty state $(0, 0)$) is chosen uniformly at random. 
The value of $B$, the workload to be completed, is only reduced if the action
is active. If the scheduler reaches the state ($T = 1$, $B > 0$), the job cannot 
be finished on time, and an extra penalty
$F(B_n^i-a_n^i)=0.2(B_n^i-a_n^i)^2$ is incurred. In addition, activating the arms 
involves a fixed cost $c = 0.8$. The rewards are given by
\begin{equation}
r_{n}^i(X_{n}^i, U_{n}^i, c) =
\begin{cases}
(1 - c)U_{n}^i, & \text{if } B_n^i > 0 \text{ and } T_n^i > 1, \\
(1 - c)U_{n}^i - F(B_n^i - a_n^i), & \text{if } B_n^i > 0 \text{ and } T_n^i = 1, \\
0, & \text{otherwise.}
\end{cases}
\end{equation}
The authors of \cite{Yu} provide an explicit expression for Whittle index, given by:
\begin{equation}
\lambda(T, B, c) =
\begin{cases}
0, & \text{if } B = 0, \\
1 - c, & \text{if } 1 \leq B \leq T - 1, \\
F(B - T - 1) - F(B - T) + 1 - c, & \text{if } T \leq B.
\end{cases}
\end{equation}

As the state space is large we use DQN to learn the Lagrangian indices. A neural network with 3 hidden layers having (512, 256, 128) neurons connected through the ReLU activation is used. A batch size of 32 is used and the size of the experience replay buffer is 1000. The exploration decay factor is 0.9995 and the learning rate for the Adam optimiser is $10^{-5}$.  

We first consider the homogeneous case using $N$ = 5 and $M$ = 2, with a processing cost $c = 0.8$.
The results are presented in Figure~\ref{fig:DQNhomo}.

\begin{figure}[H]
  \centering

  \begin{subfigure}[t]{0.45\textwidth}
    \centering
    \includegraphics[width=\textwidth]{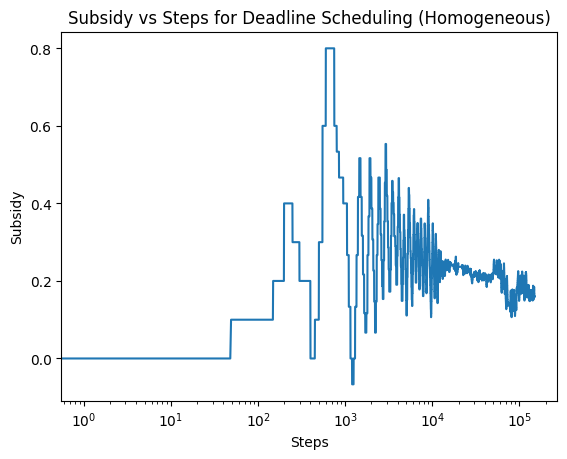}
    \caption{Subsidy (Lagrange multiplier) estimation.}
    \label{fig:RestartSubsidy1}
  \end{subfigure}
  \begin{subfigure}[t]{0.45\textwidth}
    \centering
    \includegraphics[width=\textwidth]{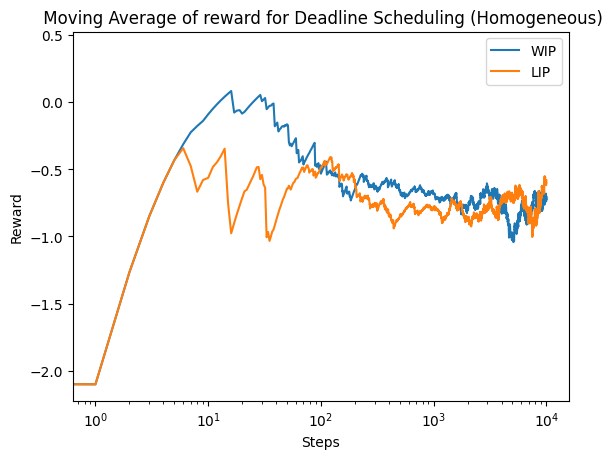}
    \caption{Moving average reward.}
    \label{fig:RestartSubsidy2}
  \end{subfigure}
  \caption{Deadline scheduling problem (homogeneous arms). Algorithm~\ref{tab:DQN}.}
  \label{fig:DQNhomo}
\end{figure}

Then we consider the scenario of heterogeneous arms, where each arm has a state space of size $|S|= 130$ and $N = 20$. 
The results are presented in Figure~\ref{fig:DQNheter}.
To differentiate the arms, we assign different activation cost values $c = \{0.1, 0.3, 0.6, 0.8\}$ for every 5 arms. 
As was suggested in Section~\ref{sec:dqn}, we take into account heterogeneity by increasing the
dimension of the input to the neural network. The first dimension is used to give information about the type 
of arm and the other dimensions represent the state of the arm. Specifically, the values of the first input are $0,1,2,3$ for $c=0.1$, $c=0.3$, $c=0.6$ and $c=0.8$.

\begin{figure}[H]
  \centering

  \begin{subfigure}[t]{0.45\textwidth}
    \centering
    \includegraphics[width=\textwidth]{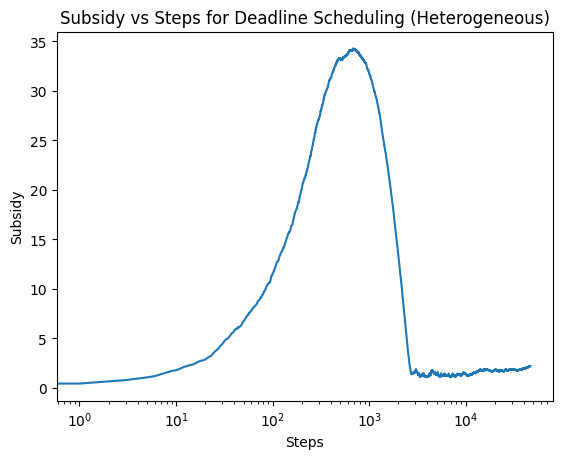}
    \caption{Subsidy (Lagrange multiplier) estimation.}
    \label{fig:RestartSubsidy1}
  \end{subfigure}
  \begin{subfigure}[t]{0.45\textwidth}
    \centering
    \includegraphics[width=\textwidth]{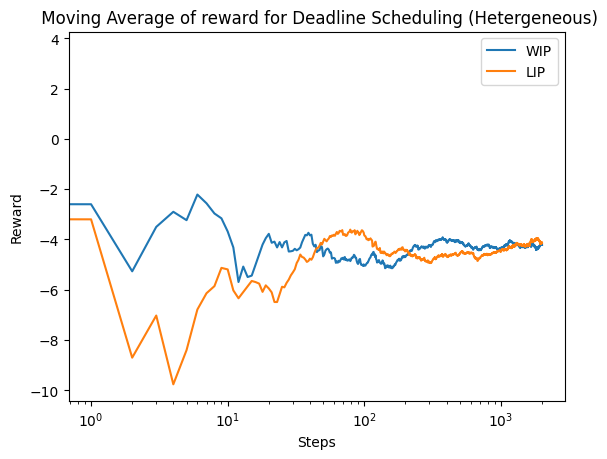}
    \caption{Moving average reward.}
    \label{fig:RestartSubsidy2}
  \end{subfigure}
  \caption{Deadline scheduling problem (heterogeneous arms). Algorithm~\ref{tab:DQN}.}
  \label{fig:DQNheter}
\end{figure}

Here, as was the case in the restart model, the performance of LIP is comparable to the performance of WIP.
However, since we do not need to use a reference state in our algorithms for LIP as opposed to the algorithms
from \cite{Avrachenkov2,Pagare,Robledo1,Robledo2}, the implementation of the reinforcement learning algorithms
for LIP is much more straightforward than that for WIP and at the end uses less number of computer operations.
This leads to the gain in time and memory.

\section{Asymptotic Optimality}
\label{sec:asympt}

\subsection{Outline} In this section we provide an alternative proof of asymptotic optimality of the Lagrangian index
for the case when the number of homogeneous arms goes to infinity and the global attractor hypothesis \cite{Gast,Verloop} holds. First, let us briefly overview the proof technique. 
\begin{enumerate}
\item We start by two key ingredients of our argument:
\begin{enumerate}
\item A well known result regarding control of constrained Markov decision processes with a single constraint, which allows us to assume that the optimal policy can be taken to  use randomization in at most one state between at most two controls \cite{Ross}. 

\item The notion of exchangeability of an infinite collection of random variables and its consequences \cite{Kingman}, \cite{Olshen}. 
\end{enumerate}

\item The (a), (b) above are together applied  to the infinite collection of arms, each a Markov decision process viewed as an $\ZA := (\X\times\U)^\infty$-valued random variable, thus rendering the entire collection a $\ZA^\infty$-valued random variable. Specifically, an index is assigned to infinite sequences of states at each time belonging to a set of probability $1$ using what we call the `Procedure 2', in a manner that renders the infinite array  exchangeable. 

\item The  above procedure leads to infinitely many Markov decision processes with their controlled transition probabilities coupled through their additional dependence on a random probability measure on $\X$.

\item This fact allows us to denote the resulting parametrized transition matrix as $P(\nu)$ for a random probability measure $\nu \in \PA(\X)$. Considering the map $\Psi: \mu \in \PA(\X) \mapsto \mu P(\mu) \in \PA(\X)$ where $\mu \in \PA(\X)$ is treated as a row vector, we make the following important assumption, a variant of the \textit{Global Attractor Hypothesis} introduced in \cite{Gast}, \cite{Verloop}.

\item The foregoing paves the way for the rest of the analysis that follows.
\end{enumerate}

The  Global Attractor Hypothesis alluded to above is as follows and will be assumed to hold throughout.\\

\noindent \dg \ \textbf{Global Attractor Hypothesis:} 
 The map $\Psi$ has a unique fixed point $\nu^*$ and the iteration $\mu_{n+1} = \mu_nP(\mu_n), n \geq 0,$ converges to $\nu^*$ for any choice of $\nu_0 \in \PA(\X)$.\\
 
In our work, $\nu, \mu$ in bullet 4 above will be replaced by the empirical state frequency of the infinite family of controlled Markov chains. 

\subsection{Some Background Material}

Here we recall some results about constrained Markov decision processes and exchangeable random variables.

\begin{enumerate}

\item \textit{Constrained Markov decision processes:} Consider a single Markov decision process with a finite state space $\X$ and a finite (more generally, a compact) action space $\U$ with controlled transition probabilities $p(j|i,u)$ for $i,j \in \X$ and $u \in \U$. Assume that it is irreducible under every randomized Markov policy $\varphi(\cdot|i) \in \PA(\U), i \in \X$ with the corresponding transition probabilities $p_\varphi(j|i) := \sum_u\varphi(u|i)p(j|i,u)$ for $i,j \in \X$ and the unique stationary distribution $\pi_\varphi \in \PA(\X)$. Then a constrained Markov decision process for average cost criterion with a single constraint has the form:\\

\noindent Maximize over $\varphi$ the cost $\sum_{i,u}\pi_\varphi(i)p_\varphi(u|i)r_1(i,u)$ subject to the constraint 
$$\sum_{i,u}\pi_\varphi(i)p_\varphi(u|i)r_2(i,u) \leq K$$ 
for prescribed primary and secondary reward functions $r_1, r_2 : S\times\U \to \mathbb{R}$.\\

A remarkable result of \cite{Ross} shows that this problem has an optimal control that requires randomization at at most one state between at most two different control choices. See \cite{BorkarConstrained} and \cite{BorkarConstraint} for a convex analytic proof using a result of Dubins \cite{Dubins} about extreme points of intersections of closed convex sets with half spaces, which also generalizes the above result to multiple constraints.

\item \textit{Exchangeable random variables:} An infinite sequence of Polish space valued random variables is said to be exchangeable if its joint law is invariant under finite permutations. This has many interesting consequences. Define the \textit{exchangeable} $\sigma$-field as the collection of sets in the product Borel $\sigma$-field that are invariant under finite permutations. (It is easy to check from the definition that this is indeed a $\sigma$-field.) Then we have:\\

\begin{theorem}  (\textit{De Finetti's theorem}) An exchangeable sequence of random variables is i.i.d.\ when conditioned on the exchangeable $\sigma$-field. \end{theorem}

This has many important consequences. The one that concerns us most is the following `strong law of large numbers for exchangeable random variables'.

\begin{corollary}\label{SLLNX} Let $Y_n, n \geq 1,$ be integrable exchangeable random variables in $\mathbb{R}^d, d \geq 1$.  Then there exists an $\mathbb{R}^d$-valued integrable random variable $\widetilde{Y}$ measurable with respect to the associated exchangeable $\sigma$-field such that
$$\frac{1}{n}\sum_{m=1}^nY_m \to \widetilde{Y}, \ \mbox{a.s.}$$
\end{corollary}
\begin{proof} This follows by combining the preceding theorem with the classical strong law of large numbers.
\end{proof}

For a Polish space $S$, consider an exchangeable family $\{X_n\}$ of $S$-valued random variables. Then for   a  measurable function $f: S \to \R^d$, $\{Y_n = f(X_n)\}$ is also exchangeable and the above can be applied to $\{Y_n\}$ if $E[\|f(Y_n)\|] < \infty$.
\end{enumerate}


Since in our problem the arms are identical, by symmetry one could divide both the displays in the statement of ($P_1$) by $N$ 
and reduce the problem to a single constrained problem of maximizing the stationary expectation of $r(X_n,U_n)$ subject to the stationary 
expectation of  $I\{U_n=1\}$ being equal to $\frac{M}{N}$. This symmetry implies in particular that $\lambda_*$ 
is also the Lagrange multiplier for the single arm constrained MDP. This MDP, however, is a complete artifice, 
but it will play a role in our proof. We assume that this MDP is irreducible under all stationary policies. 
From the theory of constrained MDPs \cite{Altman,Ross,Piunov}, we have the following.

\begin{theorem}\label{Ross} 
For a single constraint MDP, there exists an optimal stationary randomized policy $\varphi_*(\cdot|\cdot)$ that needs to randomize in at most one state, between at most two control choices. \end{theorem}

\begin{remark} Defining `occupation measures' $\mu(\cdot,\cdot)$ for a Markov chain controlled by a stationary randomized policy by $\mu(i,u) :=$ the stationary probability of $\{X_n = i, U_n = u\}$ for $i \in \X, u \in \U$, such $\mu$ form a bounded convex polytope $C_1$ whose extreme points correspond to stationary non-randomized policies. A constraint of the form $\sum_{i,u}\mu(i,u)f(i,u) \leq C$ for some $f: \X\times\U \to \mathbb{R}$ and $C \in \mathbb{R}$ defines another convex set of probability measures $C_2$. Using a theorem of Dubins \cite{Dubins}, it can be shown that the extreme points of $C_1\cap C_2$ correspond to the $\mu$ that can be expressed as a convex combination of at most two stationary non-randomized policies and can be realized by at most one randomization between two control choices at at most one state. This approach to the result of \cite{Ross} is taken in \cite{BorkarConstrained}, \cite{BorkarConstraint}. The latter also gives an extension to the case of multiple constraints. From this viewpoint, it is clear that this optimal policy may be non-unique if $C_2$ just touches $C_1$ along one of its faces of dimension $1$ or more, in which case the set of optimal $\mu$ is a convex set that is not a singleton. However, this situation is non-generic because it can be destroyed by a small perturbation of the function $f$.
\end{remark}

\subsection{Proof of Asymptotic Optimality}

Since $|\U| = 2$ in our case, the latter `at most' in Theorem~\ref{Ross} means that it needs to randomize at most in one state, say $x_*$,  between $0$ and $1$ with probabilities $1-\beta,\beta$ resp., for some $\beta\in [0,1]$. Let $\A:=\{x \in \X : \varphi_*(1|x)=1\}$ denote the set of states that are always active under this policy. Let $\pi_*$ denote the unique stationary distribution of $\{X_n\}$ under the optimal policy. Since the constraint \eqref{constraint} is met under the optimal policy $\varphi_*$, we have
\begin{equation}
\pi_*(\A) + \beta\pi_*(\{x_*\}) = \alpha := \frac{M}{N}. \label{optconstraint}
\end{equation}
In the following, we shall compare two policies:
\begin{enumerate}
\item the relaxed optimal policy, denoted by OPT and given by $\varphi^*$, and,
\item the Lagrangian index policy, LIP, that renders active at time $n$ the chains $i, 1 \leq i \leq N,$ with the top $M$ values of $\gamma(X^i_n)$, breaking any ties with equal probability.
\end{enumerate}
We are interested in the $N\uparrow\infty$ limit.  

Recall the definition of $\alpha$ in the original set up. From now on, we take $\alpha$ as a prescribed  fraction in $(0,1)$ and define:\\

\noindent  $M_N :=$ the smallest integer exceeding $\alpha N$, \\

\noindent $m_N := M_N-1 =$ the largest integer not exceeding $\alpha N$. \\

Replace $M$ by $M_N$, resp., $m_N$,  in the foregoing. Clearly,
$$\alpha_N := \frac{M_N}{N} \to \alpha, \ \alpha_N' := \frac{m_N}{N} \to \alpha \ \mbox{as} \ N\uparrow\infty.$$
Under OPT, the chains are independent of each other. Let $\epsilon > 0$. Under OPT, \eqref{optconstraint} and the strong law of large numbers imply that
$$\lim_{N\to\infty}\frac{1}{N}\sum_{i=1}^N\Big(I\{X^i_n\in \A\} + I\{X^i_n=i^*, U^i_n=1\}\Big) = \alpha \ \mbox{a.s.} $$
for all $n \geq 0$. \\

Under LI, the chains are \textit{not} independent, but they form an exchangeable family. This is because, by our choice of the index rule and the tie-breaking protocol, their joint laws are invariant under permutations. We shall work directly with the infinite product chain $[X^1_n, X^2_n, \cdots]$, $n \geq 0$, in $\X^\infty$.\\

Henceforth, we work with the infinite product chain, i.e., $N = \infty$.\\

Number the states according to decreasing values of the Lagrangian  index $\gamma(j) := Q_{\lambda^*}(j,1) - Q_{\lambda^*}(j,0)$. Let 
\begin{eqnarray*}
\X^+ &:=& \{j\in \X: \gamma(j) > 0\}, \\ 
\X^0 &:=&  \{j\in \X: \gamma(j) = 0\},\\
\X^-  &:=&  \{j\in \X: \gamma(j) < 0\}.
\end{eqnarray*} 
By Theorem \ref{Ross}, we can choose an optimal policy such that $|\X^0| \leq 1$. We do so henceforth. The following is immediate from Theorem \ref{Ross}.

\begin{theorem}\label{singlechain} For a single chain, the optimal control will be assigned as per the following procedure:\\

\noindent {\rm\textbf{Procedure 1:}}
\textit{Assign $1$ for $j \in \X^+$, $0$ for $j\in \X^-$, and $1$ with probability $\beta$ and $0$ with probability $1-\beta$ for $i\in \X^0$ if $|\X_0| = 1$. Here $\beta \in [0,1]$ is such that if $\pi$ denotes the stationary distribution of the chain, then $\pi(\X^+) + \beta\pi(\X^0) = \alpha$. } 
\end{theorem}

Let $X^i_n, i\ge 1,$ denote the current states of the arms. Order the arms as $1, 2,...,$ according to decreasing values of $\gamma(X^i_n)$'s. We follow the next procedure, which is the counterpart of `Procedure 1' above for the infinite sequence: \\

\noindent \textbf{Procedure 2:} \textit{Club together the $X^i_n$'s with equal indices into a single group. Rank these groups according to decreasing $\gamma(X^i_n)$. 
Define $\xi(k)$ as the asymptotic fraction of the number of chains belonging to the group with the index $\gamma(k)$.
This is well defined by the strong law of large numbers for exchangeable random variables.
Assign control $1$ to the top $\ell^*$ such groups, where\\
}
$$\ell^* := \max\{k \geq 1: \sum_{l\le k} \xi(l) < \alpha \}.$$

\textit{We assign control $1$ to each chain in state $\ell^*+1$ in i.i.d.\ fashion with probability $\beta$.
Then, for chains in group $\ell^*+1$, an asymptotic fraction $\kappa := \alpha - \sum_{l\le \ell^*} \xi(l)$ 
is assigned control $1$, the rest are assigned the control $0$ in an exchangeable manner. The remaining groups are assigned control $0$.}\\

%
%

\begin{theorem}\label{indep} The arms are asymptotically independent (as $N\uparrow\infty$) under the above Procedure 2 and the Global Attractor Hypothesis \textbf{$(\dagger)$}, i.e., for any fixed $K$ and the arms $\{X^i_n, 1 \leq i \leq K \}$, $E\left[\prod_{i=1}^Kf_i(X^i_n)\right] = \prod_{i=1}^KE\left[f_i(X^i_n)\right]$ \ $\forall \ f_i\in C_b(\X), 1 \leq i \leq K < \infty$, $\forall n$.
\end{theorem}

\begin{proof}
Recall that for each $i \geq 1$, $X^i_n, n \geq 0,$ is a controlled Markov chain controlled by the process $U^i_n, n \geq 0$.
Let $\E$ denote the $\sigma$-field of exchangeable events. Conditioned on $\E$,  $\{(X^i_n, U^i_n), n \geq 0\}, i \geq 1,$ are independent processes by de Finetti's theorem. Then almost surely,  the random probability measure $\nu_n \in \PA(\X\times\U)$ defined by
$$\nu_n(j,u) := \lim_{N\to\infty}\frac{1}{N}\sum_{i=1}^NI\{X^i_n=j, U^i_n = u\}$$
is well defined by the strong law of large numbers for exchangeable random variables, viz., Corollary \ref{SLLNX}. 
Necessarily, the random variable $\nu_n \in \PA(\X\times\U)$ is $\E$-measurable. \\

We shall write any $\kappa \in \PA(\X\times\U)$ as
$\kappa(i,u) = \bar{\kappa}(i)\varphi_\kappa(u|i) , \ i\in \X, u \in \U.$ Also define the corresponding transition probabilities as $p_\kappa(j|i) = \sum_up(j|i,u)\varphi_\kappa(u|i), \ i,j \in \X.$ In particular, disintegrating $\nu_n(x\times u)$ as $\bar{\nu}_n(x)\varphi_\nu(u|x)$, the conditional law $\varphi_\nu( \cdot | \cdot)$ is completely specified by Procedure~2 given $\bar{\nu}_n$. In particular, this allows us to write $p_{\nu_n}(j|k)$ as $p_{\bar{\nu}_n}(j|k)$ by slight abuse of notation. 
Thus, we can define
\begin{equation}
P(\bar\nu_n)= [p_{\bar\nu_n}(j|k)]_{j,k}. \label{newnotation}
\end{equation}
This notation will be important in what follows.\\

Suppose we apply Procedure 2 above. 
The procedure implies that the transition probabilities of the arms are coupled only through the dependence of the control policy on $\bar{\nu}_n$. To see this, note that a control $u \in \{0,1\}$ is assigned to $i$ by Procedure 2 depending solely on the asymptotic fraction of states (as $N\uparrow\infty$) belonging to $\gamma^{-1}(u)$ and an independent randomization if needed. Thus, the outcome of the procedure depends solely on the probability vector $\xi$ in Procedure 2, which at time $n$, corresponds to $\bar{\nu}_n$. The latter in turn is clearly $\E$-measurable. Therefore, conditioned on $\E$, for each $i$, $\{X^i_n\}$ is a Markov chain with the transition probability matrix independent of $i$ (since they are conditionally i.i.d.), but a function of $\bar{\nu}_n$ which is $\E$-measurable. \\

In addition, a reasoning analogous to the one used for defining $\nu_n$ also shows that for any $k \in \X$,
\begin{eqnarray*}
0 &=&\lim_{N\uparrow\infty}\frac{1}{N}\sum_{i=1}^N\Big(I\{X^i_{n+1} = k\} - P(X^i_{n+1} = k|X^i_n, U^i_n)\Big) \\
&=& \lim_{N\uparrow\infty}\left(\frac{1}{N}\sum_{i=1}^NI\{X^i_{n+1} = k\} - \frac{1}{N}\sum_{i=1}^N\sum_{j\in S}p_{\bar{\nu}_n}(k|j)I\{X^i_n=j\}\right) \\
&=& \bar{\nu}_{n+1}(k) - \sum_{j\in S}p_{\bar{\nu}_n}(k|j)\bar{\nu}_n(j)  \   \mbox{a.s.}, 
\end{eqnarray*}
where the first equality follows from the strong law of large numbers for martingales with bounded increments. 
Thus $\{\bar{\nu}_n\}$ satisfies the iteration
\begin{equation}
\bar{\nu}_{n+1} = \bar{\nu}_nP(\bar{\nu}_n), \ n \geq 0. \label{nuiteration}
\end{equation}

Under our \textit{Global Attractor Hypothesis} \textbf{$(\dagger)$},  $\bar{\nu}_n \to \nu^*$ a.s., where $\nu^*$ is the unique deterministic fixed point of the map $\Psi$. Hence at stationarity (i.e., as $n\to\infty$), the $\{(X^i_n, U^i_n), n \geq 0\}$ for different $i$'s are coupled only through $\nu_*$ which is deterministic. This implies their independence.
\end{proof}

\begin{theorem} Procedure 2 is optimal with respect to stationary policies for the infinite arms with the reward
$$E\left[\lim_{n\to\infty}\frac{1}{n}\sum_{m=0}^{n-1}\left(\lim_{N\to\infty}\frac{1}{N}\left(\sum_{i=1}^{N}r^i(X^i_m, U^i_m)\right)\right)\right].$$ 
\end{theorem}

\begin{proof} By Theorem \ref{indep}, the arms are asymptotically independent under Procedure 2.  Also, the  average reward with respect to  $n\uparrow\infty$ limit is well defined by the strong law of large numbers in view of Theorem \ref{indep}. This allows us to work with the infinite product chain, i.e., with $N = \infty$, henceforth. The marginal laws of $\{(X^i_n, U^i_n)\}$, $n \geq 0$, are independent of $i,n$ at stationarity, i.e., as $n\to\infty$, by the preceding theorem.  Therefore the expected average rewards for all arms are equal. From the construction, for each arm in isolation, the procedure reduces to Procedure 1 which is optimal by choice. The claim follows. \end{proof}

We have so far been considering the infinite arms case directly. Our final result justifies why it is the correct limiting case for $N$ bandits with the Lagrangian index policy, as $N\to\infty$.

Recall that as $N$ varies, it does not make sense to speak of making $\alpha N$ arms active and the rest passive, because $\alpha N$ may not be an integer. Therefore, as $N$ increases, we consider the problem of choosing $M_N := \lceil\alpha N\rceil$ arms active and the rest passive. A similar treatment is possible with $m_N := \lfloor\alpha N\rfloor$ or a prescribed combination of the two options. 

First we make precise the sense in which we take this limit. Bandits with $N$ arms specify a probability measure on $(\X\times\U)^N$, say $\zeta_N$. For $1 \leq K \leq N \leq \infty$, denote by $\zeta_{N,K}$ its marginal on the first $K$ copies of $\X\times\U$, i.e., on $(\X\times\U)^K$.  We say that $\zeta_N \in \PA((\X\times\U)^N) \to \zeta_\infty \in \PA((\X\times\U)^\infty)$ if for each $K \leq N$, $\zeta_{N,K} \to \zeta_{\infty,K}$ in $\PA((\X\times\U)^K)$.
    
We are now in a position to state and prove our final result.

\begin{theorem} As $N \to \infty$, the laws of $N$-armed bandits with the Lagrangian index policy converge to the set of laws of bandits with infinite arms that follow Procedure 2 in the sense defined above.
\end{theorem}
\begin{proof} Fix $n \geq 1$. Recall that the Lagrangian index policy for $N < \infty$ bandits is as follows.  At any time instant, rank the observed states of the bandits according to decreasing values of their Lagrangian indices, breaking any ties uniformly. Then render the top $M_N$ of them active, the rest passive. We shall now analyze the $N\to\infty$ limit of this scheme. For this, fixe $1 \leq M < \infty$ and consider the restrictions of the laws of restrictions the infinite product chains to the product of the first $M$ factor spaces. Since the latter space is compact, the probability measures on this space for each fixed $M$ are tight and therefore converge along a subsequence. By a standard diagonal argument, we pick a further subsequence along which this is true for all $M \geq 1$. 

Next, note the following:
\begin{enumerate}

\item By construction, the fraction of active arms, i.e., $\frac{M_N}{N} \to \alpha$ sample pathwise.

\item Exchangeability is preserved in the limit.    

\item For each $N$, the Lagrangian index of any active arm is $\geq$ the Lagrangian index of any passive arm. This property is preserved in the limit.
\end{enumerate}
It is easy to see that the three together lead to Procedure 2 in the limit. 
\end{proof}

\begin{remark} We have passed to the $N\to\infty$ limit first in order to facilitate working with the infinite product chain, which facilitates the use of exchangeability. 
 Note that we do not claim uniqueness of the limit law as $N\to\infty$.
The argument that follows is concerned largely with $n\to\infty$, i.e., the approach of the infinite product chain to stationarity, where the Global Attractor Hypothesis allows us to establish asymptotic independence. The uniqueness of the limiting law as $n \to \infty$ crucially depends on the fact that $\nu_n \to$ a deterministic limit $\nu^*$, as already seen. This proof strategy allows us to avoid persistent  $N\to\infty$ limits throughout the argument once they are put out of the way. It is, however, possible to keep finite $N$ in sight throughout, thanks to the limit theorem proved in \cite{Diaconis}, but then the arguments will be far more cumbersome. 
\end{remark}

\section{Conclusion}
\label{sec:concl}

We have proposed and analyzed the Lagrangian index for average cost restless bandits. 
It is computationally more tractable than Whittle index when the latter is not explicitly known 
and does not require Whittle indexability. We have calculated analytically the Lagrangian index
for the restart model. We have provided several numerical experiments. LIP is seen to perform 
comparably well on problems that are Whittle indexable and as well or better on problems that are not. 
We also prove its asymptotic optimality in the infinite arms limit under the popular `global 
attractor hypothesis' using a novel argument based on exchangeability and de Finetti's theorem.

\bigskip

\section{Appendix: Overview of Stochastic Approximation}

Stochastic approximation,  introduced by Robbins and Monro \cite{Robbins},  refers to the iteration in $\R^d, d \geq 1,$ given by\footnote{This is not the most general formulation, but it will suffice for our purposes.}
\begin{equation}
x(n+1) = x(n) + a(n)\left[h(x(n)) + M(n+1) + \varepsilon(n+1)\right], \ n \geq 0. \label{SA0}
\end{equation}
Here $\{M_n\}$ is a square-integrable martingale difference sequence in $\R^d$ with respect to the increasing $\sigma$-fields $\F_n := \sigma(x(0), M(k), 1 \leq k \leq n)$ for $n \geq 0$. That is,   $E\left[M(n+1)|\F_n\right] = 0$ componentwise $\forall \ n \geq 0$.
Assume:
\begin{enumerate}
\item $h: \R^d \to \R^d$ is Lipschitz,

\item $\{M_n\}$  satisfies
\begin{equation}
E\left[\|M(n+1)\|^2|\F_n\right] \leq K\left(1 + \|x(n)\|^2\right) \ \forall n \label{mgbound}
\end{equation}
for some constant $K > 0$,

\item $\{\varepsilon(n)\}$ is a bounded random sequence in $\R^d$ adapted to $\{\F_n\}$ satisfying $\varepsilon(n) \to 0$ a.s.
\end{enumerate}

Below, we shall consider the classical framework above, the `two time scale stochastic approximation', and asynchronous stochastic approximation, in that order.

\begin{enumerate}
\item \textbf{Single time scale stochastic approximation}\\

We shall assume that a.s.,
\begin{equation}
\sup_n\|x(n)\| < \infty. \label{supbound}
\end{equation}
That is, the iterates remain bounded a.s., with a possibly random bound. This is usually ensured  by using problem-specific techniques such as a suitable stochastic Liapunov argument. 
One approach to analyze this iteration is the ODE  (for `Ordinary Differential Equations') approach, developed over the years by many authors \cite{Benaim, Der, Ljung, Metivier}  which treats (\ref{SA0}) as a noisy discretization of the ODE
\begin{equation}
\dot{x}(t) = h(x(t)) \label{ODE0}
\end{equation}
and argues that the iterates $\{x(n)\}$ a.s.\ track the asymptotic behaviour of $x(t)$ as $t\to\infty$. The argument goes as follows. Define the \textit{algorithmic time scale} as follows: Define $t(n), n \geq 0,$ by setting $t(0) = 0$ and $t(n+1) = t(n) + a(n), n \geq 0$. Define the interpolated iterates $\bar{x}(t), t \geq 0,$ by $\bar{x}(t(n)) = x(n), n \geq 0,$ with linear interpolation on $[t(n), t(n+1)], n \geq 0$. Fix $T > 0$. Let $x^s(t), t \in [s, s + T],$ the solution to (\ref{ODE0}) with initial condition $x^s(s) = \bar{x}(s)$ for $s \geq 0$. Then one can show that a.s.,
\begin{equation}
\lim_{s\to\infty}\sup_{t \in [s,s+T]}\|\bar{x}(t) - x^s(t)\| \to 0. \label{ODElim} 
\end{equation}
This is proved as follows. The cumulative errors in the above ODE approximation on the interval $[s,s+T]$ are of three types:
\begin{itemize}
\item  Those that are due to discretization  tend to zero a.s.\ because they are 
$$O\left(\sum_{k \geq t(n)}a(k)^2\right).$$ 

\item Those due to $\{\varepsilon(n)\}$ tend to zero a.s.\ because $\varepsilon(n)$ does.

\item Those due to the martingale difference noise $\{M(n)\}$ go to zero because 
\begin{equation}\label{vanishingtail}
\sum_{k\geq n}a(k)M(k+1) \to 0,
\end{equation}
a.s. The latter is a consequence of the martingale convergence theorem which ensures that under (\ref{mgbound}) and (\ref{supbound}), the martingale $Z(n) := \sum_{m=0}^na(m)M(m+1)$ converges a.s.\ as $n\to\infty$. 
\end{itemize}
Then (\ref{ODElim}) follows by a simple application of the Gronwall inequality.

From (\ref{ODElim}), one can deduce that a.s., $x(n) \to$ an \textit{internally chain transitive} invariant set of (\ref{ODE0}) \cite{Benaim}. This is the most general characterization of the asymptotics of (\ref{SA0}), see \textit{ibid.} for the definitions and further details. We shall deal with the simplest case - for us this set will be an equilibrium point.

\item \textbf{Two time scale stochastic approximation}\\

With the above backdrop, we now describe the two time scale stochastic approximation \cite{BorkarBook} given by
\begin{eqnarray}
x(n+1) &=& x(n) + a(n)\Big[h(x(n),y(n)) + M_1(n+1) \nonumber \\
&& \ \ \ \ \ \ \ \ \ \ \ \  + \ \varepsilon_1(n+1)\Big], \label{SAfast} \\
y(n+1) &=& y(n) + b(n)\Big[g(x(n),y(n)) + M_2(n+1) \nonumber \\
&& \ \ \ \ \ \ \ \ \ \ \ \ + \ \varepsilon_2(n+1)\Big]. \label{SAslow}
\end{eqnarray}
Here $\{x(n)\}, \{y(n)\}$ are iterates in $\R^d, \R^s$ resp., $h: \R^{d+s} \to \R^d$, $g: \R^{d+s} \to \R^s$ are Lipschitz,  $\{M_i(n)\}, i = 1,2,$ are resp.\ $d$ and $s$-dimensional martingale difference sequences with  respect to the $\sigma$-fields $\F_n := \sigma(x(0), y(0), M_i(m), 1 \leq m \leq n, i= 1,2)$, $n \geq 0$, satisfying 
\begin{equation}\label{twoscalebound}
E\left[\|M_1(n+1)\|^2 + \|M_2(n+1)\|^2 | \F_n\right] \leq K\left(1 + \|x(n)\|^2 + \|y(n)\|^2\right)
\end{equation}
for a suitable constant $K > 0$. The $\varepsilon_i(n), n \geq 1, i = 1,2,$ are bounded random variables in $\R^d, \R^s$ resp., adapted to $\{\F_n\}$, that tend to $0$ a.s.\ as $n\to\infty$. We assume the counterpart of (\ref{supbound}):
\begin{equation}\label{a.s.bound}
\sup_n\left(\|x(n)\| + \|y(n)\|\right) < \infty \ \mbox{a.s.}
\end{equation}

Treating $\{a(n)\}$ as the base step size sequence, (\ref{SAslow}) can be rewritten as
$$y(n+1) = y(n) + a(n)\left[\left(\frac{b(n)}{a(n)}\right)g(x(n),y(n)) + \varepsilon_2(n+1) + M_2(n+1)\right].$$
Since $\frac{b(n)}{a(n)} \to 0$, the ODE that will be tracked then is
$$\dot{x}(t) = h(x(t), y(t)), \ \dot{y}(t) = 0.$$
That is, $\bar{x}(t), t \in [s,s+T],$ tracks the ODE (in the sense of (\ref{ODElim}) ) given by
$$\dot{x}^s(t) = h(x^s(t), y^s(s)).$$
Suppose the ODE
$$\dot{x}(t) = h(x(t),y)$$
with $y$ treated as a constant parameter, has a unique asymptotically stable equilibrium $\kappa(y)$ for some Lipschitz $\kappa : \R^s \to \R^d$. Then a technical argument similar to that for (\ref{SA0}) above leads to $$x(n) - \kappa(y(n)) \to 0 \ \mbox{a.s.}$$
Then (\ref{SAslow}) can be rewritten as
\begin{eqnarray*}
y(n+1) &=& y(n) + b(n)\Bigg[g(\kappa(y(n)),y(n)) \nonumber \\
&& \ \ \ \ \ \ \ \ \ \ \ + \ \tilde{\varepsilon}_2(n+1) + M_2(n+1)\Bigg],
\end{eqnarray*}
where
$$\tilde{\varepsilon}_2(n+1) := \varepsilon_2(n+1) + g(x(n), y(n)) - g(\kappa(y(n)), y(n)) \to 0 \ \mbox{a.s.}$$
Suppose the ODE
$$\dot{y}(t) = g(\kappa(y(t)), y(t))$$
has a unique asymptotically stable equilibrium $y^*$. Then by familiar arguments, $y(n) \to y^*$ a.s. Then $x(n) \to \kappa(y^*)$ a.s.

\item \textbf{Asynchronous stochastic approximation}\\

Here the idea is that different components of (\ref{SA0}) are updated by different processors or agents, each equipped with its own clock. We shall stick to the special case here, wherein only one, viz., the $(Z_n)$-th component is  updated at time n. Then
$$\nu(i,n) := \sum_{m=0}^nI\{Z_m = i\}, \ n \geq 0,$$
denotes the local clock on which the $i$-th component is updated. The key condition we need is that
$$\liminf_{n\to\infty}\frac{\nu(i,n)}{n} > 0 \ \mbox{a.s.}$$
That is, all $i$ are sampled a strictly positive fraction of times. Suppose we replace the step size $a(n)$ by $a(\nu(Z_n,n))$ for all $n$. With this in place, under Assumptions 2.3 and 2.4 of \cite{Abounadi}, one can argue as in \cite{Abounadi} that the iterates asymptotically track the ODE
$$\dot{x}(t) = \frac{1}{d}h(x(t))$$
with $d$ being the dimension of $x(t)$.
This is a time-scaled version of (\ref{ODE0}), i.e., has the same trajectories slowed down by a factor of $d$, therefore with the same asymptotic behaviour. 

The advantage of using $a(\nu(Z_n,n))$ instead of $a(n)$ is that in the distributed implementation, a processor need know only the local clock. Also, when $a(n)$ is used, a slower processor suffers further by having the convergence slowed down further by the faster decay of step sizes. On the flip side, the former entails additional computation for keeping track of the local count, which, however, is often relatively minor.

\end{enumerate}

\section*{Acknowledgements} The authors would like to thank E. Altman and A. Piunovskiy
for helpful discussions about constrained Markov decision processes.

\section*{Statements and Declarations} The authors declare no competing or conflicting interests related to the present work.

\end{document}